\documentclass[10pt]{article}
\usepackage[accepted]{tmlr}
\usepackage[utf8]{inputenc}
\usepackage{amsmath,amssymb,amsfonts}
\usepackage{geometry}
\usepackage{amsthm}
\usepackage{caption}
\usepackage{subcaption}
\usepackage{tikz,url}
\usepackage{mathtools}
\mathtoolsset{showonlyrefs}
\usepackage{dirtytalk}
\usepackage{hyperref}

\newtheorem{theorem}{Theorem}
\newtheorem{lemma}[theorem]{Lemma} 

\newtheorem{remark}[theorem]{Remark}

\newtheorem{corollary}[theorem]{Corollary}

\newcommand\E{\mathbb{E}}
\newcommand\R{\mathbb{R}}
\newcommand\dx{\mathrm{d}}

\theoremstyle{plain}

\DeclareMathOperator*{\argmax}{arg\,max}
\DeclareMathOperator*{\argmin}{arg\,min}
\newcommand{\tT}{\mathrm{T}}

\title{Conditional Generative Models are Provably Robust: Pointwise Guarantees for Bayesian Inverse Problems}
\author{\name Fabian Altekr\"uger \email fabian.altekrueger@hu-berlin.de \\
\addr Department of Mathematics \\
Humboldt-Universit\"at zu Berlin
\AND
\name Paul Hagemann \email hagemann@math.tu-berlin.de \\
\addr Institute of Mathematics \\
Technische Universit\"at Berlin
\AND
\name Gabriele Steidl \email steidl@math.tu-berlin.de \\
\addr Institute of Mathematics \\
Technische Universit\"at Berlin
}

\def\month{07}  
\def\year{2023} 
\def\openreview{\url{https://openreview.net/forum?id=Wcui061fxr}} 

\begin{document}

\maketitle

\begin{abstract}
    Conditional generative models became a very powerful tool to sample from Bayesian inverse problem posteriors. 
     It is well-known in classical Bayesian literature that posterior measures are quite robust with respect to perturbations of both the prior measure and the negative log-likelihood, which includes perturbations of the observations. However, to the best of our knowledge, the robustness of conditional generative models with respect to perturbations of the observations has not been investigated yet. In this paper, we prove for the first time that appropriately learned conditional generative models provide robust results for single observations.
\end{abstract}

\section{Introduction}

Recently, the use of neural networks (NN) in the field of uncertainty quantification has emerged, as one often is interested in the statistics of a solution and not just in point estimates. In particular, Bayesian approaches in inverse problems received great interest.
In this paper, we are interested
in learning the whole posterior distribution in Bayesian inverse problems by conditional generative NNs as proposed, e.g., in \citep{adler_deep,ardizzone2019guided, cond_score, Hagemann_2022}. Addressing the posterior measure instead of end-to-end reconstructions has several advantages as illustrated
in Figure~\ref{fig:posterior_density}. More precisely, if we consider a Gaussian mixture model as prior distribution and a linear forward operator with additive Gaussian noise, then the posterior density (red) can be computed explicitly. Obviously, these curves change smoothly with respect to the observation $y$, i.e., we observe a continuous behaviour of the posterior also with respect to observations near zero. In particular, (samples of) the posterior can be used to provide additional 
information on the reconstructed data, for example on their uncertainty.
As can be seen in the figure, for fixed $y$ the minimum mean squared error (MMSE) estimator
delivers just an averaged value. In contrast to the non robust maximum a-posteriori (MAP) estimator, this is not the one with the highest probability. Even worse, the MMSE can output values which are not even in the prior distribution. For a more comprehensive comparison of MAP estimator, MMSE estimator and posterior distribution, we refer to Appendix~\ref{app:example}

Several robustness guarantees on the posterior were proved in the literature. One of the first results in the direction of stability with respect to the distance of observations was obtained by \citet{stuart_2010} with respect to the Hellinger distance, see also \citet{DS2017}. 
A very related question instead of perturbed observations concerns the approximations of forward maps, which was investigated by \citet{MX2009}. Furthermore, different prior measures were considered in \citep{Hosseini2017, HN2017,Sullivan_2017}, where they also discuss the general case in Banach spaces. Two recent works \citep{latz_wellposedness,Sprungk_2020} investigated the (Lipschitz) continuity of the posterior measures with respect to a multitude of metrics, where \cite{latz_wellposedness} focused on the well-posedness of the Bayesian inverse problem and \cite{Sprungk_2020}  on the local Lipschitz continuity. 
Most recently, in \citet{garbunoinigo2023bayesian} the stability estimates have been generalized to integral probability metrics circumventing some Lipschitz conditions done in \citet{Sprungk_2020}.
Our paper is based on the findings of \citet{Sprungk_2020}, but relates them with conditional generative NNs that aim to learn the posterior.

\begin{figure*}
\centering
\begin{subfigure}{.245\textwidth}
  \includegraphics[width=\linewidth]{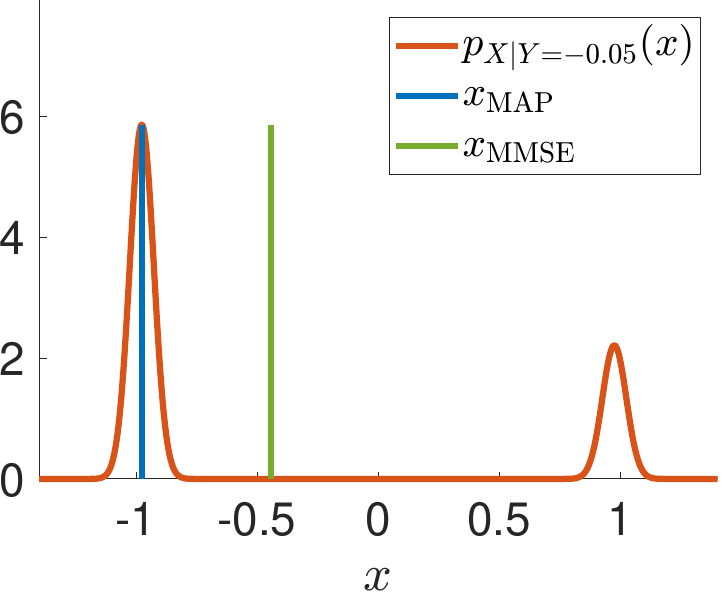}
  \caption*{$y=-0.05$}
\end{subfigure}%
\hfill
\begin{subfigure}{.245\textwidth}
  \includegraphics[width=\linewidth]{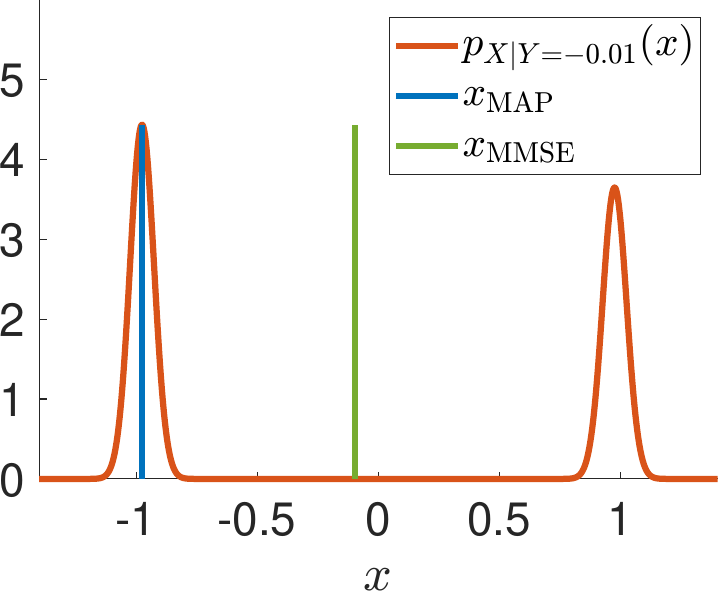}
    \caption*{$y=-0.01$}
\end{subfigure}%
\hfill
\begin{subfigure}{.245\textwidth}
  \includegraphics[width=\linewidth]{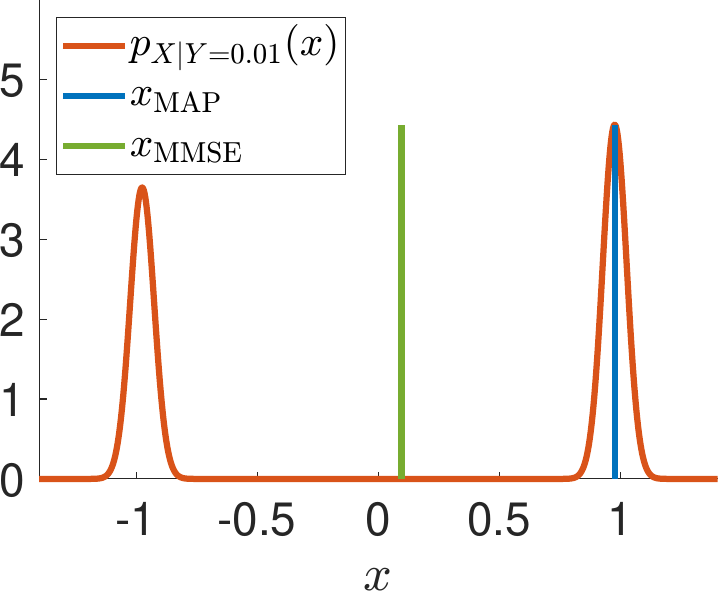}
    \caption*{$y=0.01$}
\end{subfigure}%
\hfill
\begin{subfigure}{.245\textwidth}
  \includegraphics[width=\linewidth]{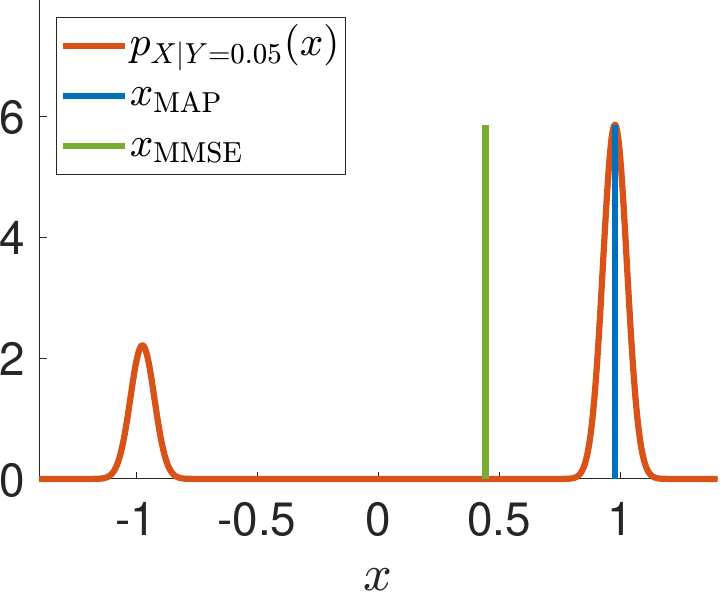}
    \caption*{$y=0.05$}
\end{subfigure}%
\caption{Posterior density (red), MAP estimator (blue) and MMSE estimator (green) for different observations $y=-0.05,-0.01,0.01,0.05$ 
(from left to right). While the MAP estimator is discontinuous with respect to the observation $y$,  the posterior density 
is continuous with respect to y. 
The MMSE estimator just gives the expectation value of the posterior which is, in contrast to MAP, not the value with highest probability.
} \label{fig:posterior_density}
\end{figure*}

More precisely, in many machine learning papers, 
the following idea is pursued in order to solve inverse problems simultaneously for all observations $y$: 
Consider a family of generative models $G_\theta(y,\cdot)$ with parameters $\theta$, 
which are supposed to map a latent distribution, like the standard Gaussian one, to the absolutely continuous posteriors 
$P_{X|Y=y}$, i.e., $G_{\theta}(y,\cdot)_{\#} P_Z \approx P_{X|Y=y}$. 
In order to learn such a conditional generative model, usually a loss of the form 
$$
L(\theta) \coloneqq \mathbb{E}_{y \sim P_Y} [D (P_{X|Y=y}, G_\theta(y,\cdot)_{\#} P_Z)]
$$
is chosen with some \say{distance} $D$ between measures 
like the Kullback-Leibler (KL) divergence $D = \text{KL}$ used in \citet{ardizzone2019guided} for training normalizing flows
or the Wasserstein-1 distance $D = W_1$ appearing, e.g., in the framework of (conditional) Wasserstein generative adversarial networks (GANs) \citep{adler_deep,arjovsky2017wasserstein, liu2021wasserstein}. Also conditional diffusion models \citep{eq_diff,song2021scorebased,tashiro2021csdi} fit into this framework. Here \cite{debortoli2022convergence}
showed that the standard score matching diffusion loss also optimizes the Wasserstein distance between the target and predicted distribution. 

However, in practice we are usually interested in the reconstruction quality from a single or just a few measurements
which are null sets with respect to $P_Y$. 
In this paper, we are interested in the important question, whether there exist any guarantees for the NN output to
be close to the posterior for one specific measurement $\tilde{y}$.
Our main result in Theorem \ref{thm:posterior_bound} 
shows that for a NN learned such that the loss becomes small in the Wasserstein-1 distance, say
$L(\theta) < \varepsilon$, 
the distance $W_1(P_{X|Y=\tilde y}, G_\theta(\tilde y,\cdot)_{\#} P_Z)$ becomes also small for the single observation $\tilde y$.
More precisely, we get the bound
$$W_1(P_{X|Y=\tilde y}, G_\theta(\tilde y,\cdot)_{\#} P_Z) \le C \varepsilon^\frac{1}{n+1},$$ 
where $C$ is a constant and $n$ is the dimension of the observations.
To the best of our knowledge, this is the first estimate given in this direction.

We like to mention that in contrast to our paper, 
where we assume that samples are taken from the distribution for which the NN was learned,
\citet{robust_cond_ood} observed that conditional normalizing flows 
are unstable when feeding them out-of-distribution observations. 
This is not too surprising given some literature on the instability of (conditional) normalizing flows \citep{understanding_and_mitigating,norm_flows_ood}.

\paragraph{Outline of the paper.}
The main theorem is shown in Section~\ref{sec:main}.
For this we introduce several lemmata for the local Lipschitz continuity of posterior measures and conditional generative models with respect to the Wasserstein distance. 
In Section~\ref{sec:cond_gen_models}, we discuss the dependence of our derived bound on the training loss for different conditional generative models.
In Appendix~\ref{app:example},
we illustrate by a simple example with a Gaussian mixture prior and Gaussian noise,
why posterior distributions can be expected to be more stable than maximum a-posteriori (MAP) estimations and have more desirable properties than minimum mean squared error (MMSE) estimations.

\section{Pointwise Robustness of Conditional Generative NNs} \label{sec:main}
Let $X \in \R^m$ be a continuous random variable with law $P_X$  determined by its density function $p_X$
and $f \colon \mathbb R^m \rightarrow \mathbb R^n$ a measurable function.
We consider a Bayesian inverse problem 
\begin{align} \label{inverse_prob}
Y = \mathrm{noisy} (f(X))  
\end{align}
where "noisy" describes the underlying noise model. A typical choice is additive Gaussian noise, resulting in
\begin{align}
Y = f(X) + \Xi, \quad \Xi \sim \mathcal N(0, \sigma^2 I_n).
\end{align}
Let $G_\theta = G \colon \R^n \times \R^d \to \R^m$ be a conditional generative model 
trained to approximate the posterior distribution $P_{X|Y=y}$ using the latent random variable $Z \in \R^d$.
We will assume that all appearing measures are absolutely continuous and that the first moment of $G(y,\cdot)_{\#}P_Z$ is finite for all $y \in \R^n$. 
In particular, the posterior density is related via Bayes' theorem through the prior $p_X$ and the likelihood $p_{Y|X=x}$ as
$$p_{X|Y=y} \propto p_{Y|X=x} p_X,$$
where $\propto$ means quality up to a multiplicative normalization constant. Since the posterior density $p_{X|Y=y}$ is only almost everywhere unique, we choose the unique continuous representative.
Further, we assume that the negative log-likelihood $- \log p_{Y|X=x}$ is bounded from below with respect to $x$,
i.e., $\inf_x - \log p_{Y|X=x} > - \infty$ and that the prior has finite second moment. In particular, this includes 
mixtures of additive and multiplicative noise $Y = f(X) + \Xi_1 + \Xi_2 f(X)$, if $X$, $\Xi_1$ and $\Xi_2$ are independent, or log-Poisson noise commonly arising in computerized tomography.

We will use the Wasserstein-1 distance \citep{villani2009optimal}, which is a metric on the space of probability measures with finite first moment and is defined 
 for measures $\mu$ and $\nu$ on the space $\R^m$ as
 \begin{align}
W_1(\mu,\nu) = \inf_{\pi \in \Pi (\mu, \nu)} \int_{\R^m \times \R^m} \Vert x - y \Vert d\pi(x,y),
\end{align}
where $\Pi (\mu,\nu)$ contains all measures on $\R^m \times \R^m$ with $\mu$ and $\nu$ as its marginals.
The Wasserstein distance can be also rewritten by its dual formulation \citep[Remark 6.5]{villani2009optimal} as
\begin{align} \label{eq:w1_dual}
W_1(\mu, \nu) = \max_{\mathrm{Lip}(\varphi) \leq 1} \int \varphi(x) \dx (\mu-\nu)(x).
\end{align}

First, we show the local Lipschitz continuity of our generating measures $G(y,\cdot)_\# P_Z$ with respect to $y$, where we assume a local boundedness of the Jacobian of the generator with respect to the observation.

\begin{lemma}[Local Lipschitz continuity of generator] \label{gen_cont}
For all $r>0$, there exists some $L_r >0$ such that for any parameterized family of generative models $G$ with
$\Vert \nabla_y G(y,z) \Vert \leq L_r$ for all $z \in \mathrm{supp}(P_Z)$ and all $y \in \R^n$ with $\Vert y \Vert \le r$ it holds 
$$W_1(G(y_1,\cdot)_{\#} P_Z, G(y_2,\cdot)_{\#} P_Z) \leq L_r \Vert y_1-y_2 \Vert$$ for all $y_1,y_2 \in \R^n$ with $\Vert y_1 \Vert, \Vert y_2 \Vert \le r$.
\end{lemma}

\begin{proof}
We use the mean value theorem which yields for every $z \in \mathrm{supp}(P_Z)$ and all $y_1,y_2 \in \R^n$ with $\Vert y_1 \Vert, \Vert y_2 \Vert \le r$
\begin{align*}
\Vert G(y_1,z) - G(y_2,z) \Vert 
&= \Big\Vert \int_0^1 \nabla_y G (y_1 + t(y_2 - y_1),z) ( y_1 - y_2) \dx t  \Big\Vert\\
&\le  \int_0^1 \Vert \nabla_y G (y_1 + t(y_2 - y_1),z) \Vert \dx t   \Vert y_1 - y_2   \Vert \\
&\le L_r \Vert y_1 - y_2  \Vert.
\end{align*}
Next, we apply the dual formulation of the Wasserstein-1 distance to estimate
\begin{align*}
W_1(G(y_1,\cdot)_{\#} P_Z, G(y_2,\cdot)_{\#} P_Z)  &= \max_{\mathrm{Lip}(\varphi) \le 1} \E_{z \sim P_Z} [\varphi ( G (y_1,z)) -  \varphi (G(y_2,z))] \\
&\le \max_{\mathrm{Lip}(\varphi) \le 1} \E_{z \sim P_Z} [\vert \varphi ( G (y_1,z) )  - \varphi ( G(y_2,z)) \vert] \\
&\le \E_{z \sim P_Z} [\Vert G(y_1,z) - G(y_2,z) \Vert] \\
&\le L_r \Vert y_1 - y_2 \Vert. 
\end{align*}
\end{proof}

We want to highlight that there is a trade-off between regularity of the generator due to a small Lipschitz constant \citep{GFPC2021,MKKY2018} and expressivity of the generator requiring a large Lipschitz constant \citep{Hagemann_2021,SBDD2022}.

\begin{remark}
If $P_Z$ is supported on a compact set, then the assumption 
in Lemma~\ref{gen_cont} is fulfilled for generators which are, e.g., continuously differentiable and then it follows from the extreme value theorem. 
Note that if we choose $P_Z$ to be a Gaussian distribution, then it holds $\mathrm{supp}(P_Z) = \R^d$. Thus, for continuously differentiable generators 
it is not clear that this assumption is fulfilled, but at least the weaker assumption $\Vert \nabla_y G(y,z) \Vert \leq L_r$ for all $z \in \R^d$ with $\Vert z \Vert \le \tilde r$ and all $y \in \R^n$ with $\Vert y \Vert \le r$ holds true. In this case, we can show that Lemma~\ref{gen_cont} holds true up to an arbitrary small additive constant, see Appendix~\ref{app:infinite_support} for more details.
\end{remark}

By the following lemma, which is just \citep[Corollary 19]{Sprungk_2020} for Euclidean spaces, 
the local Lipschitz continuity of the posterior distribution
with respect to the Wasserstein-1 distance is guaranteed under the assumption of a locally Lipschitz likelihood.

\begin{lemma}[Local Lipschitz continuity of the posterior] \label{post_cont}
Let the forward operator $f$ and the likelihood $p_{Y|X=x}$ in \eqref{inverse_prob} be measurable. 
Assume that there exists a function $M \colon[0,\infty)\times \mathbb{R} \to [0,\infty )$ 
which is monotone in the first component and non-decreasing in the second component 
such that for all $y_1,y_2 \in \R^n$ with $\Vert y_1 \Vert , \Vert y_2 \Vert \le r$ for $r>0$ 
and for all $x \in \R^m$ it holds
\begin{align} \label{eq:likelihood_bounded}
\vert \log p_{Y|X=x}(y_2) - \log p_{Y|X=x}(y_1) \vert \le M (r,\Vert x \Vert ) \Vert y_1 - y_2 \Vert.
\end{align}
Furthermore, assume that $M(r,\Vert \cdot \Vert ) \in L^2_{P_X} (\R^m, \mathbb{R})$. 
Then, for any $r > 0$ there exists a constant $C_r < \infty$ such that for all $y_1,y_2 \in \R^n$ 
with $\Vert y_1 \Vert , \Vert  y_2 \Vert \le r$ we have
\begin{align*}
W_1(P_{X|Y=y_1}, P_{X|Y=y_2}) \le C_r \Vert y_1-y_2 \Vert.
\end{align*}
\end{lemma}
The Lipschitz constants of the family of generative models and the posterior distributions $P_{X|Y=y}$ can be related to each other under some convergence assumptions. Let the assumptions of Lemma~\ref{post_cont} be fulfilled, assume further that there exists a family of generative models $(G^\varepsilon)_\varepsilon$ satisfying
\begin{align*}
\lim_{\varepsilon \rightarrow 0} G^{\varepsilon}(y,\cdot)_{\#} P_Z = P_{X|Y=y}
\end{align*}
with respect to the $W_1$-distance and consider observations $y_1,y_2 \in \R^n$ with $\Vert y_1 \Vert, \Vert y_2 \Vert \le r$.
Then, by the triangle inequality it holds
\begin{align*}
\lim_{\varepsilon \rightarrow 0} W_1(G^{\varepsilon}(y_1,\cdot)_{\#}P_Z,G^{\varepsilon}(y_2,\cdot)_{\#}P_Z) 
&\leq \lim_{\varepsilon \rightarrow 0}  W_1(G^{\varepsilon}(y_1,\cdot)_{\#}P_Z,P_{X|Y = y_1}) +
W_1(P_{X|Y = y_1},P_{X|Y = y_2}) \\
&\quad +  W_1(P_{X|Y = y_2}, G^{\varepsilon}(y_2,\cdot)_{\#}P_Z) \\
&= W_1(P_{X|Y = y_1},P_{X|Y = y_2}) \\
&\leq C_r \Vert y_1 - y_2 \Vert.
\end{align*}
Hence, under the assumption of convergence, we expect the Lipschitz constant of our conditional generative models to behave similar to the one of the posterior distribution. 

\begin{remark}
The assumption \eqref{eq:likelihood_bounded} is for instance fulfilled 
for additive Gaussian noise $\Xi \sim \mathcal{N}(0, \sigma^2 \mathrm{Id})$. 
In this case 
\begin{align*}
-\log p_{Y|X=x}(y)= 
\frac{n}{2} \log (2 \pi \sigma^2)+ \frac{1}{2\sigma^2} \Vert y - f(x) \Vert^2.
\end{align*}
Hence $-\log p_{Y|X=x}(y)$
is differentiable with respect to $y$ and we get local Lipschitz continuity of the negative log-likelihood. 
\end{remark}

Now we can prove our main theorem
which ensures  pointwise bounds on the Wasserstein distance between posterior and generated measure, if the expectation over $P_Y$ is small. In particular, the bound depends on the local Lipschitz constant of the conditional generator with respect to the observation which may depend on the architecture of the generative model, the local Lipschitz constant of the inverse problem, the expectation over $P_Y$ and the probability of the considered observation $\Tilde{y}$. We want to highlight that the bound depends on the evidence $p_Y(\tilde y)$ of an observation $\tilde y$ and indicates that we generally cannot expect a good pointwise estimate for out-of-distribution observations, i.e., $p_Y(y) \approx 0$. This is in agreement with the empirical results presented in \citet{robust_cond_ood}. The proof of Theorem~\ref{thm:posterior_bound} has a nice geometric interpretation, which is visualized in Figure~\ref{fig:sketch_proof}.

\begin{theorem} \label{thm:posterior_bound}
Let the forward operator $f$ and the likelihood $p_{Y|X = x}$ in \eqref{inverse_prob} fulfill the assumptions of Lemma \ref{post_cont}. 
Let $\tilde{y} \in \R^n$ be an observation with $p_Y(\Tilde{y}) = a > 0$. 
Further, assume that $y \mapsto p_Y(y)$ is differentiable with $\Vert \nabla p_Y(y) \Vert \leq K$ for $K>0$ and all $y \in \R^n$. 
For fixed $k = \frac{a}{2K} + \Vert \tilde y \Vert$, 
assume that we have trained a family of generative models $G$ which fulfills 
$\Vert \nabla_y G(y,z) \Vert \leq L_k$ for all $z \in \mathrm{supp}(P_Z)$ and all $y \in \R^n$ with $\Vert y \Vert \le k$ for some $L_k > 0$
such that 
\begin{align} \label{eq:expectation_W1}
\mathbb{E}_{y \sim P_Y}  [ W_1(P_{X|Y=y}, G(y,\cdot)_{\#}P_Z) ] \leq \varepsilon
\end{align}
for some $\varepsilon > 0$. 
Then we have for $\varepsilon \le 1$ that
\begin{align} \label{estimate}
W_1(P_{X|Y=\tilde{y}}, G(\tilde{y},\cdot)_{\#}P_Z) \leq (L_{k} + C_{k}) \frac{\varepsilon^{\frac{1}{n+1}}a}{2 K}  + \frac{2 \varepsilon^{\frac{1}{n+1}}}{S_n (\frac{a}{2 K})^n a},
\end{align}
where $S_n \coloneqq \pi^{\frac{n}{2}} / \Gamma (\frac{n}{2} + 1)$ and 
$C_\bullet$ is the Lipschitz constant from Lemma \ref{post_cont}.
If $\varepsilon$ additionally satisfies 
$\varepsilon 
\le 
\left( \frac{a}{2K}\right)^{n+1} \frac{(L_{k}+C_{k})S_n a}{2n}$,
then it holds
\begin{align} \label{estimate1}
W_1(P_{X|Y=\tilde{y}}, G(\tilde{y},\cdot)_{\#}P_Z) &
\le 
(L_{k}+  C_{k}) ^{1-\frac{1}{n+1}} (1 + \frac1n) \left( \frac{2n}{S_n a}\right)^\frac{1}{n+1} \varepsilon^\frac{1}{n+1}.
\end{align}

\end{theorem}

\begin{proof}
Let $ 0 < r \le \frac{a}{2 K}$. 
Then, for $y \in B_r(\tilde{y})$, there exists by the mean value theorem some $\xi \in \overline{y \tilde y}$ such that
\begin{align*}
|p_Y(y) - p_Y(\tilde{y}) | \le \Vert \nabla p_Y(\xi) \Vert  \Vert y - \Tilde{y} \Vert \le K r \le  \frac{a}{2}.
\end{align*}
Consequently, each $y \in B_r(\tilde{y})$ has at least probability $p_Y(y) \geq \frac{a}{2}$. 
Moreover, by the volume of the $n$-dimensional ball it holds that 
\begin{align}
P_Y(B_r(\tilde{y})) = \int_{B_r (\tilde y)}  p_Y(y) \dx y 
\ge  \frac{\pi^{\frac{n}{2}}}{\Gamma (\frac{n}{2} + 1)} r^n \frac{a}{2} = S_n r^n \frac{a}{2}.
\end{align}
Now we claim that there exists $\widehat y \in B_r(\tilde{y})$ with 
\begin{align} \label{eq:bound_W1_yhat}
W_1(P_{X|Y=\widehat y}, G(\widehat y, \cdot)_{\#}P_Z) \leq \frac{2 \varepsilon}{S_n r^n a}.
\end{align}
If this would not be the case, this would imply a contradiction to \eqref{eq:expectation_W1} by
\begin{align*}
\mathbb{E}_{y \sim P_Y}  [ W_1(P_{X|Y=y}, G(y,\cdot)_{\#}P_Z) ]  &= \int_{\mathbb{R}^n} W_1(P_{X|Y=y}, G(y,\cdot)_{\#}P_Z) \dx P_Y(y)  \\
&\geq \int_{B_r({\Tilde{y}})} W_1(P_{X|Y=y}, G(y,\cdot)_{\#}P_Z) \dx P_Y(y) \\
&> \int_{B_r({\Tilde{y}})} \frac{2 \varepsilon}{S_n r^n a} \dx P_Y(y) \\
&= P_Y(B_r(\Tilde{y})) \frac{2 \varepsilon}{S_n r^n a} 
\ge \varepsilon.
\end{align*}
Next, we show the local Lipschitz continuity of $y \mapsto W_1(P_{X|Y=y}, G(y,\cdot)_{\#}P_Z)$ on $B_r (\tilde{y})$ 
by combining Lemma~\ref{gen_cont} and Lemma~\ref{post_cont}. 
Let $y_1,y_2 \in B_r(\tilde y)$, so that  $\Vert y_1 \Vert, \Vert y_2 \Vert \le \Vert \tilde y \Vert + r$.
Let $L_{\Vert \tilde y \Vert + r}>0$ be the local Lipschitz constant from Lemma~\ref{gen_cont} 
and $C_{\Vert \tilde y \Vert + r}$  the local Lipschitz constant from Lemma~\ref{post_cont}. 
Using the triangle inequality and its reverse, we get
\begin{align} \label{eq:localLipschitz_objective}
&\quad| W_1(P_{X|Y=y_1}, G(y_1,\cdot)_{\#} P_Z) - W_1(P_{X|Y=y_2}, G(y_2,\cdot)_{\#} P_Z) |  \nonumber \\
&\le | W_1(P_{X|Y=y_1}, G(y_1,\cdot)_{\#} P_Z) - W_1(P_{X|Y=y_1}, G(y_2,\cdot)_{\#} P_Z) | \nonumber \\ 
&\quad + | W_1(P_{X|Y=y_1}, G(y_2,\cdot)_{\#} P_Z) - W_1(P_{X|Y=y_2}, G(y_2,\cdot)_{\#} P_Z) | \\ 
&\le W_1(G(y_1,\cdot)_{\#} P_Z, G(y_2,\cdot)_{\#} P_Z) + W_1(P_{X|Y=y_1}, P_{X|Y=y_2}) \nonumber \\
&\le (L_{\Vert \tilde y \Vert + r} + C_{\Vert \tilde y \Vert + r}) \Vert y_1 - y_2 \Vert \nonumber .
\end{align}
Combination of the results in \eqref{eq:bound_W1_yhat} and \eqref{eq:localLipschitz_objective} yields the estimate
\begin{align*}
W_1(P_{X|Y=\tilde{y}}, G(\tilde{y},\cdot)_{\#}P_Z) &\le |W_1(P_{X|Y=\tilde{y}}, G(\tilde{y},\cdot)_{\#}P_Z) - W_1(P_{X|Y=\widehat y}, G(\widehat y,\cdot)_{\#}P_Z) | \\
&\quad + | W_1(P_{X|Y=\widehat y}, G(\widehat y,\cdot)_{\#}P_Z) | \\
& \leq (L_{\Vert \tilde y \Vert + r} + C_{\Vert \tilde y \Vert + r}) r + \frac{2 \varepsilon}{S_n r^n a} \\
& \leq (L_{\Vert \tilde y \Vert + \frac{a}{2K}} + C_{\Vert \tilde y \Vert + \frac{a}{2 K}}) r + \frac{2 \varepsilon}{S_n r^n a}.
\end{align*}
If $\varepsilon \le 1$, we can choose
$r = \varepsilon^{\frac{1}{n+1}} \frac{a}{2K} \le \frac{a}{2K}$ which results in \eqref{estimate}. 
On the other hand, the radius $r$, for which the right-hand side becomes minimal, is given by 
$$r = \Big(\frac{2 n \varepsilon}{(L_{\Vert \tilde y \Vert + \frac{a}{2K}} + C_{\Vert \tilde y \Vert + \frac{a}{2 K}}) S_n a} \Big)^{\frac{1}{n+1}}.$$
Plugging this in, we get \eqref{estimate1}, which has the same asymptotic rate..
However, we need that $r \le \frac{a}{2K}$ which implies 
$$
\varepsilon 
\le 
\left( \frac{a}{2K}\right)^{n+1} \frac{(L_{\Vert \tilde y \Vert + \frac{a}{2K}}+C_{\Vert \tilde y \Vert + \frac{a}{2 K}})S_n a}{2n}.
$$
\end{proof}

\begin{figure}[t!]
\centering
\begin{subfigure}[t]{.65\textwidth}
\includegraphics[width=\linewidth]{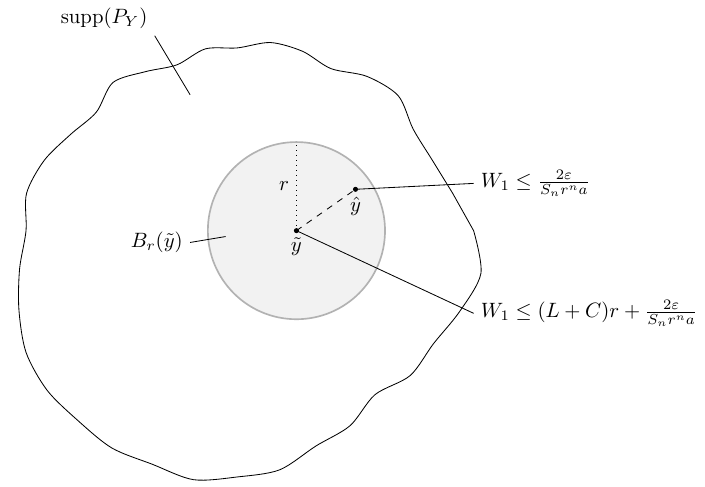}
\end{subfigure}
\caption{Geometric interpretation of the proof of Theorem~\ref{thm:posterior_bound}. Inside the ball $B_r(\tilde y)$ we can find some $\hat y$, for which the Wasserstein distance $W_1 (P_{X|Y=\hat{y}},G(\hat{y},\cdot)_{\#}P_Z)$ is bounded by $\frac{2\varepsilon}{S_n r^n a}$. Using the regularity of the generator and of the inverse problem, the Wasserstein distance $W_1 (P_{X|Y=\tilde{y}},G(\tilde{y},\cdot)_{\#}P_Z)$ can be bounded by the triangle inequality.}
\label{fig:sketch_proof}
\end{figure}

Note that in \eqref{eq:expectation_W1} we assume that the expectation over $P_Y$ of the Wasserstein distance is small. When training a generator, usually a finite training set is available. The measure $P_Y$ can be approximated by the empirical training set with a rate $n^{-1/d}$ on compact sets, where $n$ is the size of the training set and $d$ the dimension, see, e.g., \cite{WB2019}.

\begin{remark}
We can get rid of the dimension scaling $\varepsilon^{\frac{1}{n+1}}$ by choosing the radius as $r = \frac{a}{2K}$, which yields
\begin{align*}
W_1(P_{X|Y=\tilde{y}}, G(\tilde{y},\cdot)_{\#}P_Z) \leq (L_{\Vert \tilde y \Vert + \frac{a}{2K}} + C_{\Vert \tilde y \Vert + \frac{a}{2 K}}) \frac{a}{2K} + \frac{2 \varepsilon}{S_n (\frac{a}{2 K})^n a}.
\end{align*}
This comes at the disadvantage that the first term is constant with respect to $\varepsilon$.
\end{remark}

The following corollary provides a characterization of a perfect generative model. 
If the expectation \eqref{eq:expectation_W1} goes to zero, then for all $y \in \R^n$ with $p_Y (y) > 0$ 
the posteriors $P_{X|Y=y}$ get predicted correctly.

\begin{corollary}
\label{cor:zero}
Let the assumptions of Lemma~\ref{gen_cont} and Lemma~\ref{post_cont} hold true and assume a global Lipschitz constant in Lemma~\ref{gen_cont}. Let $p_Y$ be differentiable with $\Vert \nabla p_Y(y) \Vert \le K$ for some $K>0$ and all $y \in \R^n$. Consider a family of generative networks $(G^\varepsilon)_{\varepsilon >0}$ fulfilling 
\begin{align}
\mathbb{E}_{y \sim P_Y}  [ W_1(P_{X|Y=y}, G^{\varepsilon}(y,\cdot)_{\#}P_Z) ] \le \varepsilon
\end{align}
and assume that the Lipschitz constants $L^\varepsilon$ of $G^\varepsilon$ from Lemma~\ref{gen_cont} are bounded by some $L < \infty$. Then for all observations $y \in \R^n$ with $p_Y(y) > 0$ it holds 
\begin{align}
W_1(P_{X|Y=y}, G^{\varepsilon}(y,\cdot)_{\#}P_Z) \rightarrow 0 \quad \text{as} ~ \varepsilon \to 0.
\end{align}
\end{corollary}

\begin{proof}
We can assume that $\varepsilon \le 1$, then the statement follows immediately from Theorem~\ref{thm:posterior_bound}.
\end{proof}

Finally, we can use Theorem~\ref{thm:posterior_bound} for error bounds of adversarival attacks on Bayesian inverse problems.
Following the concurrent work of \citet{glockler2023adversarial}, an adversarial attack on a conditional generative model consists in finding a perturbation $\delta$ to the observation $\tilde{y}$ so that the prediction of the conditional generative model is as far away as possible from the true posterior, i.e.,
\begin{align*}
\delta = \argmax_{\Vert \delta \Vert \leq B} W_1(P_{X|Y = \tilde{y}},G(\tilde{y}+\delta,\cdot )_{\#} P_Z).
\end{align*}
Note that \citet{glockler2023adversarial} use the KL divergence for their adversarial attack, but for our theoretical analysis the Wasserstein distance is more suitable. 

The following corollary yields a worst case estimate on the possible attack. In the case of imperfect trained conditional generative models the attack can be very powerful depending on the strength of observation and the Lipschitz constant of the generator. If the conditional generative model is trained such that the expectation in \eqref{eq:expectation_W1} is small, then the attack can only be as powerful as the Lipschitz constant of the inverse problem allows. 
\begin{corollary}
Let the forward operator $f$ and the likelihood $p_{Y|X = x}$ in \eqref{inverse_prob} fulfill the assumptions of Lemma \ref{post_cont}. 
Let $\tilde{y} \in \R^n$ and $\delta \in \mathbb{R}^n$ with $\Vert \delta \Vert \leq B$ be an observation with $p_Y(\Tilde{y} + \delta) = a > 0$. 
Further, assume that $y \mapsto p_Y(y)$ is differentiable with $\Vert \nabla p_Y(y) \Vert \leq K$ for $K>0$ and all $y \in \R^n$. 
For fixed $k = \frac{a}{2K} + \Vert \tilde y \Vert + B$,
assume that we have trained a family of generative models $G$ which fulfills 
$\Vert \nabla_y G(y,z) \Vert \leq L_k$ for all $z \in \mathrm{supp}(P_Z)$ and all $y \in \R^n$ with $\Vert y \Vert \le k$ for some $L_k > 0$
such that 
\begin{align}
\mathbb{E}_{y \sim P_Y}  [ W_1(P_{X|Y=y}, G(y,\cdot)_{\#}P_Z) ] \leq \varepsilon
\end{align}
for some $0 < \varepsilon \le 1$. Then we have the following control on the strength of the adversarial attack
\begin{align*}
W_1(P_{X|Y = \tilde{y}},G(\tilde{y}+\delta,\cdot )_{\#} P_Z) \leq C_{\Vert \tilde{y} \Vert + B} B + (L_{k} + C_{k}) \frac{\varepsilon^{\frac{1}{n+1}}a}{2 K}  + \frac{2 \varepsilon^{\frac{1}{n+1}}}{S_n (\frac{a}{2 K})^n a}.
\end{align*}
\end{corollary}
\begin{proof}
Since it holds
\begin{align*}
W_1(P_{X|Y = \tilde{y}},G(\tilde{y}+\delta,\cdot )_{\#} P_Z) \le W_1(P_{X|Y = \tilde{y}},P_{X|Y=\tilde y + \delta}) + W_1(P_{X|Y=\tilde y + \delta},G(\tilde{y}+\delta,\cdot )_{\#} P_Z),
\end{align*}
the result follows from the application of Lemma~\ref{post_cont} and Theorem~\ref{thm:posterior_bound}.
\end{proof}

\section{Conditional Generative Models} \label{sec:cond_gen_models}
In this section, we discuss whether the main assumption, namely that the averaged Wasserstein distance $\mathbb{E}_{y \sim P_Y}  [ W_1(P_{X|Y=y}, G(y,\cdot)_{\#}P_Z) ]$ 
in \eqref{eq:expectation_W1} becomes small,
is reasonable for different conditional generative models. Therefore we need to relate the typical choices of training loss with the Wasserstein distance. For a short experimental verification in the case of conditional normalizing flows we refer to Appendix~\ref{app:experiment}.

In the following we  assume that the training loss of the corresponding models become small. This can be justified by universal approximation theorems like \citep{TITOIS2020,LCFCZGXC2022} for normalizing flows or \citep{LL2020} for GANs.

\subsection{Conditional Normalizing Flows}
\label{sec:cond_flow}
Conditional normalizing flows \citep{wpp_flows,graz_inc,ardizzone2019guided, winkler_cond_flows} are a family of normalizing flows parameterized by a condition, which in our case is the observation $y$. The aim is to learn a network $\mathcal{T} \colon \R^n \times \R^m \to \R^m$ such that $\mathcal{T}(y,\cdot)$ is a diffeomorphism and $\mathcal{T}(y,\cdot)_{\#}P_Z \approx  P_{X|Y=y}$ for all $y \in \R^n$, where $\approx$ means that two distributions are similar in some proper distance or divergence.
This can be done via minimizing the expectation on $Y$ of the \textit{forward} KL divergence $\mathbb{E}_{y \sim P_Y} [\mathrm{KL}(P_{X|Y=y},\mathcal{T}(y,\cdot)_{\#}P_Z)]$, which is equal, up to a constant, to
$$\mathbb{E}_{x \sim P_X, y \sim P_Y}[-\log p_Z(\mathcal{T}^{-1}(y,x))-\log (|\operatorname{det}D \mathcal{T}^{-1}(y,x)|)],$$ where the inverse is meant with respect to the second component, see \citet{Hagemann_2022} for more details. Training a network using the forward KL has many desirable properties like a mode-covering behaviour of $\mathcal{T}(y,\cdot)_{\#}P_Z$.  
Now conditional normalizing flows are trained using the KL divergence, while the theoretical bound in Section~\ref{sec:main} relies on the metric properties of the Wasserstein-1 distance.
Thus we need to show that we can ensure a small $\varepsilon$ in \eqref{eq:expectation_W1} when training the conditional normalizing flow as proposed.
Following \citet[Theorem 4]{on_choosing_and}, we can bound the Wasserstein distance by the total variation distance, which in turn is bounded by KL via Pinsker's inequality \citep{pinsker1963information}, i.e., 
\begin{align*}
\mathbb{E}_{y \sim P_y} [W_1(P_{X|Y=y},\mathcal{T}(y,\cdot)_{\#}P_Z)^2] &\leq
C\  \mathbb{E}_{y \sim P_Y} [\mathrm{TV}(P_{X|Y=y} ,\mathcal{T}(y,\cdot)_{\#}P_Z)^2] \\
&\leq \frac{C}{\sqrt{2}}\mathbb{E}_{y \sim P_Y}[ \mathrm{KL}(P_{X|Y=y},\mathcal{T}(y,\cdot)_{\#}P_Z)
],
\end{align*}
where $C$ is a constant depending on the support of the probability measures. However, by definition  $\mathrm{supp}(\mathcal{T}(y,\cdot)_{\#}P_Z) = \R^m$. By \citet[Lemma 4]{ADHHMS2022} the density $p_{\mathcal{T}(y,\cdot)_{\#}P_Z}$ decays exponentially. Therefore, we expect in practice that the Wasserstein distance becomes small if the KL vanishes even though \citep[Theorem 4]{on_choosing_and} is not applicable.

\subsection{Conditional Wasserstein GANs}
In Wasserstein GANs \citep{arjovsky2017wasserstein}, a generative adversarial network approach is taken in order to sample from a target distribution. For this, the dual formulation \eqref{eq:w1_dual} is used in order to calculate the Wasserstein distance between measures $P_X$ and $P_Y$. Then the 1-Lipschitz function is reinterpreted as a discriminator in the GAN framework \citep{GPMXWOCB2014}. If the corresponding minimizer in the space of 1-Lipschitz functions can be found, then optimizing the adversarial Wasserstein GAN loss directly optimizes the Wasserstein distance. The classical Wasserstein GAN loss for a target measure $\mu$ and a generator $G\colon\R^d \to \R^m$ is given by
$$\min_\theta  \max_{\mathrm{Lip}(\varphi) \leq 1} \mathbb{E}_{x \sim P_X, z \sim P_Z}[\varphi(x) - \varphi(G(z))],$$
where $d \in \mathbb N$ is the dimension of the latent space.

The Wasserstein GAN framework can be extended to conditional Wasserstein GANs \citep{adler_deep,liu2021wasserstein} for solving inverse problems. For this, we aim to train generators $G\colon \R^n \times \R^d \to \R^m$ and average with respect to the observations
$$ L(\theta) = \mathbb{E}_{y \sim P_y} \big[ \max_{\mathrm{Lip}(\varphi_y) \leq 1}  \mathbb{E}_{x \sim P_{X|Y=y}, z \sim P_Z}[\varphi_y(x) - \varphi_y(G(y,z))] \big].$$
Hence minimizing this loss (or a variant of it) directly  enforces a small $\varepsilon$ in assumption \eqref{eq:expectation_W1}.

\subsection{Conditional Diffusion Models}
In diffusion models, a forward SDE, which maps  a data distribution to an approximate Gaussian distribution is considered \citep{song2021maximum,song2021scorebased}. Then the theory of reverse SDEs \citep{ANDERSON1982313} allows to sample from the data distribution by learning the score $\nabla \log p_t(x)$, where $p_t(x)$ is the path density of the forward SDE. 
The forward SDE usually reads 
$$dX_t = - \alpha X_t \dx t + \sqrt{2 \alpha} \dx W_t,$$
while the reverse SDE is given by 
$$dY_t = - \alpha Y_t \dx t - 2\ \nabla \log p_t(x) \dx t + \sqrt{2 \alpha} \dx \tilde{W}_t,$$
where $\alpha \in \R$ describes the schedule of the SDE.
However, the path density $p_t(x)$ is usually intractable, so that the score $\nabla \log p_t(x)$ is learned with a NN $s_\theta\colon [0,T]\times \R^m \to \R^m$ such that $s_{\theta}(t,x) \approx \nabla \log p_t(x)$ for all $t \in [0,T]$ and $x \in \R^m$.
This can be ensured using the so-called score matching loss \citep{song2021scorebased} defined by
$$\min_\theta \mathbb{E}_{t \sim U([0,T]),x \sim P_{X_t}}\left[\Vert s_{\theta}(t,x) -\nabla \log p_t(x)\Vert^2\right].$$
In order to solve inverse problems, we can consider a conditional reverse SDE
\begin{align*}
dY_t = - \alpha Y_t \dx t - 2\ \nabla \log p_t(x|y) \dx t + \sqrt{2 \alpha} \dx \tilde{W}_t,
\end{align*}
where $p_t(x|y)$ is the conditional path density given an observation $y \in \R^n$. 
Consequently, we consider conditional diffusion models, where a NN $s_\theta \colon \R^n \times [0,T] \times \R^m \to \R^m$ is learned to approximate $s_{\theta}(y,t,x) \approx \nabla \log p_t(x|y)$ for all $t \in [0,T]$, $x \in \R^m$ and all observations $y \in \R^m$. Then the score matching loss for conditional diffusion models is given in \citet[Theorem 1]{cond_score} as
\begin{align} \label{eq:loss_score}
L(\theta) = \mathbb{E}_{y \sim P_Y} \big[ \mathbb E_{t \sim U([0,T]) ,x \sim P_{X_t|Y=y}} [\Vert s_{\theta}(y,t,x) -\nabla \log p_t(x|y)\Vert^2 ] \big].
\end{align}
Denote by $\Tilde{Y}$ the solution to the approximated SDE starting at $\tilde{Y}_0 \approx P_Z$ and $\tilde Y^y$ the solution of the approximated SDE conditioned on an observation $y \in \R^n$.
Then we can use the bound derived in \citet[Theorem 4]{pidstrigach2023infinitedimensional} which gives
$$\mathbb{E}_{y \sim P_Y}[W_2^2(P_{X|Y=y}, P_{\tilde{Y}_T^y})] \leq C\ \mathbb{E}_{y \sim P_Y} \big[ W_2^2\big(P_{X^y_T}, \mathcal{N}(0,\mathrm{Id})\big)\big] + C L(\theta) ,$$
where $C$ is a constant depending on the length of the interval $T$ and the Lipschitz constant of the conditional score $\nabla \log p_t(x|y)$, which we assume to be uniformly bounded (e.g. compact support of $P_Y$).
Finally, H\"olders inequality yields for the Wasserstein-1 distance
$$\mathbb{E}_{y \sim P_Y}[W_1(P_{X|Y=y}, P_{\tilde{Y}_T^y})]^2 \leq C\  \mathbb{E}_{y \sim P_Y} \big[ W_2^2\big(P_{X^y_T}, \mathcal{N}(0,\mathrm{Id})\big)\big] + C L(\theta).$$
Hence, when training the conditional diffusion model by minimizing \eqref{eq:loss_score} we also ensure that \eqref{eq:expectation_W1} becomes small. For more in depth discussion with less restrictive assumptions on the score, see also \citep{debortoli2022convergence}.

\subsection{Conditional Variational Autoencoder}

Variational Autoencoder (VAE) \citep{kingma2013auto} aim to approximate a distribution $P_X$ by learning a stochastic encoder $E_\phi \colon \R^m \to \R^d \times \R^{d,d}$ determining parameters of the normal distribution $(\mu_\phi(x),\Sigma_\phi(x))$ for $x$ sampled from $P_X$  and pushing $P_X$ to a latent distribution $P_Z$ with density $p_Z$ of dimension $d \in \mathbb N$. In the reverse direction, a stochastic decoder $D_\theta \colon \R^d \to \R^m \times \R^{m,m}$ determines parameters of the normal distribution $(\mu_\theta(z),\Sigma_\theta(z))$ for $z \in \R^d$  and pushes $P_Z$ back to $P_X$. By definition, the densities of $E_\phi$ and $D_\theta$ are given by $q_\phi(z|x) = \mathcal{N}(z;\mu_\phi(x),\Sigma_\phi(x))$ and $p_\theta(x|z) = \mathcal{N}(x;\mu_\theta(z),\Sigma_\theta(z))$, respectively.
These networks are trained by minimizing the so-called evidence lower bound (ELBO) 
\begin{align*}
\mathrm{ELBO}(\theta,\phi) =  - \mathbb E_{x \sim P_X}\big[ \mathbb E_{z \sim q_\phi(\cdot|x)} [\log (p_\theta(x|z)p_Z(z)) - \log(q_\phi(z|x)))] \big].
\end{align*}
By \citet[Theorem 4.1]{HHS2023}, the loss $L(\theta,\phi)$ is related to KL by
\begin{align}
\text{KL}(P_X,{D_\theta}_{\#}P_Z) \le \mathrm{ELBO}(\theta,\phi).
\end{align}

We can solve inverse problems by extending VAEs to conditional VAEs \citep{lim2018molecular,SLY2015} and aim to approximate the posterior distribution $P_{X|Y=y}$ for a given observation $y \in \R^n$. The conditional stochastic encoder $E_\phi \colon \R^n \times \R^m \to \R^d \times \R^{d,d}$ and conditional stochastic decoder $D_\theta \colon \R^n \times \R^d \to \R^m \times \R^{m,m}$ are trained by
\begin{align}
L(\theta,\phi) = \mathbb E_{y \sim P_Y} \big[ - \mathbb E_{x \sim P_{X|Y=y}}\big[ \mathbb E_{z \sim q_\phi(\cdot|y,x)} [\log (p_\theta(x|y,z)p_Z(z)) - \log(q_\phi(z|y,x)))] \big] \big].
\end{align}
By the same argument as above, the KL can be bounded by
\begin{align}
\mathbb E_{y \sim P_Y} [\text{KL}(P_{X|Y=y},D_{\theta}(y,\cdot)_{\#}P_Z)] \le L(\theta,\phi)
\end{align}
and, using similar arguments as in Section~\ref{sec:cond_flow}, we get the estimate
\begin{align}
\mathbb E_{y \sim P_Y} [W_1(P_{X|Y=y},D_\theta(y,\cdot)_{\#}P_Z)^2] \le \frac{C}{\sqrt{2}} L(\theta,\phi).
\end{align}

\section{Conclusion}

We showed a pointwise stability guarantee of the Wasserstein distance between the posterior $P_{X|Y=y}$ of a Bayesian inverse problem and the learned distribution $G(y,\cdot)_{\#}P_Z$ of a conditional generative model $G$ under certain assumptions. In particular, the pointwise bound depends on the Lipschitz constant of the conditional generator with respect to the observation, the Lipschitz constant of the inverse problem, the training loss with respect to the Wasserstein distance and the probability of the considered observation. 

 The required training accuracy of the bound depends on the Wasserstein-1 distance between the target distribution and the learned distribution. However, some conditional networks as the conditional normalizing flow are not trained to minimize the Wasserstein-1 distance. Consequently, a direct dependence of the bound on the training accuracy with respect to the KL divergence would be helpful. Under very strong assumptions, the continuity in Lemma \ref{gen_cont} has been shown for KL in \citet{BHKMS2023}. This could be used to derive a similar statement.
 
 Furthermore, our bound is a worst case bound and is not always practical if the constants are large. It would be interesting to check whether tightness of the bound can be shown for some examples.

\subsubsection*{Acknowledgement}
P.H. acknowledges funding by the German Research Foundation (DFG)  
within the project of the DFG-SPP 2298  ``Theoretical Foundations of Deep Learning''
and 
F.A. within project EF3-7 of Germany‘s
Excellence Strategy – The Berlin Mathematics Research Center MATH+.
The authors want to thank B. Sprungk for posing the question considered in this paper, namely
if there are guarantees that conditional generative NNs
work well for single observations.
Moreover, many thanks to J. Hertrich for fruitful discussions and for suggesting the example in the appendix.

\bibliography{refs}

\begin{thebibliography}{51}
\providecommand{\natexlab}[1]{#1}
\providecommand{\url}[1]{\texttt{#1}}
\expandafter\ifx\csname urlstyle\endcsname\relax
  \providecommand{\doi}[1]{doi: #1}\else
  \providecommand{\doi}{doi: \begingroup \urlstyle{rm}\Url}\fi

\bibitem[Adler \& Öktem(2018)Adler and Öktem]{adler_deep}
Jonas Adler and Ozan Öktem.
\newblock Deep {Bayesian} inversion.
\newblock \emph{arXiv preprint arXiv:1811.05910}, 2018.

\bibitem[Altekrüger \& Hertrich(2023)Altekrüger and Hertrich]{wpp_flows}
Fabian Altekrüger and Johannes Hertrich.
\newblock {WPPNets and WPPFlows:} {T}he power of {Wasserstein} patch priors for
  superresolution.
\newblock \emph{SIAM Journal on Imaging Sciences}, 16\penalty0 (3):\penalty0
  1033--1067, 2023.

\bibitem[Altekrüger et~al.(2023)Altekrüger, Denker, Hagemann, Hertrich,
  Maass, and Steidl]{ADHHMS2022}
Fabian Altekrüger, Alexander Denker, Paul Hagemann, Johannes Hertrich, Peter
  Maass, and Gabriele Steidl.
\newblock {PatchNR:} learning from very few images by patch normalizing flow
  regularization.
\newblock \emph{Inverse Problems}, 39\penalty0 (6), 2023.

\bibitem[Anderson(1982)]{ANDERSON1982313}
Brian~D.O. Anderson.
\newblock Reverse-time diffusion equation models.
\newblock \emph{Stochastic Processes and their Applications}, 12\penalty0
  (3):\penalty0 313--326, 1982.

\bibitem[Andrle et~al.(2021)Andrle, Farchmin, Hagemann, Heidenreich, Soltwisch,
  and Steidl]{graz_inc}
Anna Andrle, Nando Farchmin, Paul Hagemann, Sebastian Heidenreich, Victor
  Soltwisch, and Gabriele Steidl.
\newblock Invertible neural networks versus {MCMC} for posterior reconstruction
  in grazing incidence x-ray fluorescence.
\newblock In Abderrahim Elmoataz, Jalal Fadili, Yvain Quéau, Julien Rabin, and
  Loïc Simon (eds.), \emph{Scale Space and Variational Methods in Computer
  Vision}, pp.\  528–539. Springer International Publishing, 2021.

\bibitem[Ardizzone et~al.(2019)Ardizzone, L{\"u}th, Kruse, Rother, and
  K{\"o}the]{ardizzone2019guided}
Lynton Ardizzone, Carsten L{\"u}th, Jakob Kruse, Carsten Rother, and Ullrich
  K{\"o}the.
\newblock Guided image generation with conditional invertible neural networks.
\newblock \emph{arXiv preprint arXiv:1907.02392}, 2019.

\bibitem[Arjovsky et~al.(2017)Arjovsky, Chintala, and
  Bottou]{arjovsky2017wasserstein}
Martin Arjovsky, Soumith Chintala, and L{\'e}on Bottou.
\newblock Wasserstein generative adversarial networks.
\newblock In \emph{International Conference on Machine Learning}, pp.\
  214--223. PMLR, 2017.

\bibitem[Baptista et~al.(2023)Baptista, Hosseini, Kovachki, Marzouk, and
  Sagiv]{BHKMS2023}
Ricardo Baptista, Bamdad Hosseini, Nikola~B. Kovachki, Youssef~M. Marzouk, and
  Amir Sagiv.
\newblock An approximation theory framework for measure-transport sampling
  algorithms.
\newblock \emph{arXiv preprint arXiv:2302.13965}, 2023.

\bibitem[Batzolis et~al.(2021)Batzolis, Stanczuk, Schönlieb, and
  Etmann]{cond_score}
Georgios Batzolis, Jan Stanczuk, Carola-Bibiane Schönlieb, and Christian
  Etmann.
\newblock Conditional image generation with score-based diffusion models.
\newblock \emph{arXiv preprint arXiv:2111.13606}, 2021.

\bibitem[Behrmann et~al.(2021)Behrmann, Vicol, Wang, Grosse, and
  Jacobsen]{understanding_and_mitigating}
Jens Behrmann, Paul Vicol, Kuan-Chieh Wang, Roger Grosse, and Joern-Henrik
  Jacobsen.
\newblock Understanding and mitigating exploding inverses in invertible neural
  networks.
\newblock In Arindam Banerjee and Kenji Fukumizu (eds.), \emph{Proceedings of
  The 24th International Conference on Artificial Intelligence and Statistics},
  volume 130 of \emph{Proceedings of Machine Learning Research}, pp.\
  1792--1800. PMLR, 2021.

\bibitem[Dashti \& Stuart(2017)Dashti and Stuart]{DS2017}
M.~Dashti and A.~M. Stuart.
\newblock The {B}ayesian approach to inverse problems.
\newblock In \emph{Handbook of Uncertainty Quantification}, pp.\  311--428.
  Springer, 2017.

\bibitem[De~Bortoli(2022)]{debortoli2022convergence}
Valentin De~Bortoli.
\newblock Convergence of denoising diffusion models under the manifold
  hypothesis.
\newblock \emph{Transactions on Machine Learning Research}, 2022.

\bibitem[Garbuno-Inigo et~al.(2023)Garbuno-Inigo, Helin, Hoffmann, and
  Hosseini]{garbunoinigo2023bayesian}
Alfredo Garbuno-Inigo, Tapio Helin, Franca Hoffmann, and Bamdad Hosseini.
\newblock Bayesian posterior perturbation analysis with integral probability
  metrics.
\newblock \emph{arXiv preprint arXiv:2303.01512}, 2023.

\bibitem[Gibbs \& Su(2002)Gibbs and Su]{on_choosing_and}
Alison~L. Gibbs and Francis~Edward Su.
\newblock On choosing and bounding probability metrics.
\newblock \emph{International Statistical Review / Revue Internationale de
  Statistique}, 70\penalty0 (3):\penalty0 419--435, 2002.

\bibitem[Gloeckler et~al.(2023)Gloeckler, Deistler, and
  Macke]{glockler2023adversarial}
Manuel Gloeckler, Michael Deistler, and Jakob~H. Macke.
\newblock Adversarial robustness of amortized {B}ayesian inference.
\newblock In Andreas Krause, Emma Brunskill, Kyunghyun Cho, Barbara Engelhardt,
  Sivan Sabato, and Jonathan Scarlett (eds.), \emph{Proceedings of the 40th
  International Conference on Machine Learning}, volume 202 of
  \emph{Proceedings of Machine Learning Research}, pp.\  11493--11524. PMLR,
  23--29 Jul 2023.

\bibitem[Goodfellow et~al.(2014)Goodfellow, Pouget-Abadie, Mirza, Xu,
  Warde-Farley, Ozair, Courville, and Bengio]{GPMXWOCB2014}
Ian Goodfellow, Jean Pouget-Abadie, Mehdi Mirza, Bing Xu, David Warde-Farley,
  Sherjil Ozair, Aaron Courville, and Y.~Bengio.
\newblock Generative adversarial networks.
\newblock \emph{Advances in Neural Information Processing Systems}, 3, 2014.

\bibitem[Gouk et~al.(2021)Gouk, Frank, Pfahringer, and Cree]{GFPC2021}
Henry Gouk, Eibe Frank, Bernhard Pfahringer, and Michael~J. Cree.
\newblock Regularisation of neural networks by enforcing {L}ipschitz
  continuity.
\newblock \emph{Machine Learning}, 110\penalty0 (2):\penalty0 393–416, 2021.

\bibitem[Grana et~al.(2017)Grana, Fjeldstad, and Omre]{GFO2017}
Dario Grana, Torstein Fjeldstad, and Henning Omre.
\newblock Bayesian {G}aussian mixture linear inversion for geophysical inverse
  problems.
\newblock \emph{Mathematical Geosciences}, 49\penalty0 (4):\penalty0 493--515,
  2017.

\bibitem[Hagemann et~al.(2023)Hagemann, Hertrich, and Steidl]{HHS2023}
P.~Hagemann, J.~Hertrich, and G.~Steidl.
\newblock \emph{Generalized Normalizing Flows via Markov Chains}.
\newblock Elements in Non-local Data Interactions: Foundations and
  Applications. Cambridge University Press, 2023.

\bibitem[Hagemann \& Neumayer(2021)Hagemann and Neumayer]{Hagemann_2021}
Paul Hagemann and Sebastian Neumayer.
\newblock Stabilizing invertible neural networks using mixture models.
\newblock \emph{Inverse Problems}, 37\penalty0 (8), 2021.

\bibitem[Hagemann et~al.(2022)Hagemann, Hertrich, and Steidl]{Hagemann_2022}
Paul Hagemann, Johannes Hertrich, and Gabriele Steidl.
\newblock Stochastic normalizing flows for inverse problems: A {M}arkov chains
  viewpoint.
\newblock \emph{{SIAM}/{ASA} Journal on Uncertainty Quantification},
  10\penalty0 (3):\penalty0 1162--1190, 2022.

\bibitem[Hong et~al.(2022)Hong, Park, and Chun]{robust_cond_ood}
Seongmin Hong, Inbum Park, and Se~Young Chun.
\newblock On the robustness of normalizing flows for inverse problems in
  imaging.
\newblock \emph{arXiv preprint arXiv:2212.04319}, 2022.

\bibitem[Horrace(2005)]{Horrace2005}
William~C. Horrace.
\newblock Some results on the multivariate truncated normal distribution.
\newblock \emph{Journal of Multivariate Analysis}, 94\penalty0 (1):\penalty0
  209--221, 2005.

\bibitem[Hosseini(2017)]{Hosseini2017}
B.~Hosseini.
\newblock Well-posed bayesian inverse problems with infinitely divisible and
  heavy-tailed prior measures.
\newblock \emph{SIAM/ASA Journal on Uncertainty Quantification}, 5\penalty0
  (1):\penalty0 1024--1060, 2017.

\bibitem[Hosseini \& Nigam(2017)Hosseini and Nigam]{HN2017}
B.~Hosseini and N.~Nigam.
\newblock Well-posed {B}ayesian inverse problems: priors with exponential
  tails.
\newblock \emph{SIAM/ASA Journal on Uncertainty Quantification}, 5\penalty0
  (1):\penalty0 436--465, 2017.

\bibitem[Igashov et~al.(2022)Igashov, Stärk, Vignac, Satorras, Frossard,
  Welling, Bronstein, and Correia]{eq_diff}
Ilia Igashov, Hannes Stärk, Clément Vignac, Victor~Garcia Satorras, Pascal
  Frossard, Max Welling, Michael Bronstein, and Bruno Correia.
\newblock Equivariant 3d-conditional diffusion models for molecular linker
  design.
\newblock \emph{arXiv preprint arXiv:2210.05274}, 2022.

\bibitem[Kingma \& Ba(2015)Kingma and Ba]{KB2015}
Diederik~P. Kingma and Jimmy Ba.
\newblock Adam: {A} method for stochastic optimization.
\newblock In \emph{International Conference on Learning Representations}, 2015.

\bibitem[Kingma \& Welling(2013)Kingma and Welling]{kingma2013auto}
Diederik~P Kingma and Max Welling.
\newblock Auto-encoding variational bayes.
\newblock \emph{arXiv preprint arXiv:1312.6114}, 2013.

\bibitem[Kirichenko et~al.(2020)Kirichenko, Izmailov, and
  Wilson]{norm_flows_ood}
Polina Kirichenko, Pavel Izmailov, and Andrew~Gordon Wilson.
\newblock Why normalizing flows fail to detect out-of-distribution data.
\newblock In \emph{Advances in Neural Information Processing Systems}, NIPS'20,
  Red Hook, NY, USA, 2020. Curran Associates Inc.

\bibitem[Latz(2020)]{latz_wellposedness}
Jonas Latz.
\newblock On the well-posedness of {Bayesian} inverse problems.
\newblock \emph{SIAM/ASA Journal on Uncertainty Quantification}, 8\penalty0
  (1):\penalty0 451–482, 2020.

\bibitem[Lim et~al.(2018)Lim, Ryu, Kim, and Kim]{lim2018molecular}
Jaechang Lim, Seongok Ryu, Jin~Woo Kim, and Woo~Youn Kim.
\newblock Molecular generative model based on conditional variational
  autoencoder for de novo molecular design.
\newblock \emph{Journal of cheminformatics}, 10\penalty0 (1):\penalty0 1--9,
  2018.

\bibitem[Liu et~al.(2021)Liu, Zhou, Jiao, and Huang]{liu2021wasserstein}
Shiao Liu, Xingyu Zhou, Yuling Jiao, and Jian Huang.
\newblock Wasserstein generative learning of conditional distribution.
\newblock \emph{arXiv preprint arXiv:2112.10039}, 2021.

\bibitem[Lu \& Lu(2020)Lu and Lu]{LL2020}
Yulong Lu and Jianfeng Lu.
\newblock A universal approximation theorem of deep neural networks for
  expressing probability distributions.
\newblock In H.~Larochelle, M.~Ranzato, R.~Hadsell, M.F. Balcan, and H.~Lin
  (eds.), \emph{Advances in Neural Information Processing Systems}, volume~33,
  pp.\  3094--3105. Curran Associates, Inc., 2020.

\bibitem[Lyu et~al.(2022)Lyu, Chen, Feng, Cun, Zhu, Geng, Xu, and
  Chen]{LCFCZGXC2022}
Junlong Lyu, Zhitang Chen, Chang Feng, Wenjing Cun, Shengyu Zhu, Yanhui Geng,
  Zhijie Xu, and Yongwei Chen.
\newblock Para-{CF}lows: {$C^k$}-universal diffeomorphism approximators as
  superior neural surrogates.
\newblock In Alice~H. Oh, Alekh Agarwal, Danielle Belgrave, and Kyunghyun Cho
  (eds.), \emph{Advances in Neural Information Processing Systems}, 2022.

\bibitem[Marzouk \& Xiu(2009)Marzouk and Xiu]{MX2009}
Y.~Marzouk and D.~Xiu.
\newblock A stochastic collocation approach to {B}ayesian inference in inverse
  problems.
\newblock \emph{Communications in Computational Physics}, 6\penalty0
  (4):\penalty0 826--847, 2009.

\bibitem[Miyato et~al.(2018)Miyato, Kataoka, Koyama, and Yoshida]{MKKY2018}
Takeru Miyato, Toshiki Kataoka, Masanori Koyama, and Yuichi Yoshida.
\newblock Spectral normalization for generative adversarial networks.
\newblock In \emph{International Conference on Learning Representations}, 2018.

\bibitem[Pidstrigach et~al.(2023)Pidstrigach, Marzouk, Reich, and
  Wang]{pidstrigach2023infinitedimensional}
Jakiw Pidstrigach, Youssef Marzouk, Sebastian Reich, and Sven Wang.
\newblock Infinite-dimensional diffusion models for function spaces.
\newblock \emph{arXiv preprint arXiv:2302.10130}, 2023.

\bibitem[Pinsker(1963)]{pinsker1963information}
MS~Pinsker.
\newblock Information and information stability of random quantities and
  processes.
\newblock Technical report, Foreign Technology DIV Wright-Patterson AFB, Ohio,
  1963.

\bibitem[Salmona et~al.(2022)Salmona, Bortoli, Delon, and Desolneux]{SBDD2022}
Antoine Salmona, Valentin~De Bortoli, Julie Delon, and Agn{\`e}s Desolneux.
\newblock Can push-forward generative models fit multimodal distributions?
\newblock In Alice~H. Oh, Alekh Agarwal, Danielle Belgrave, and Kyunghyun Cho
  (eds.), \emph{Advances in Neural Information Processing Systems}, 2022.

\bibitem[Sohn et~al.(2015)Sohn, Lee, and Yan]{SLY2015}
Kihyuk Sohn, Honglak Lee, and Xinchen Yan.
\newblock Learning structured output representation using deep conditional
  generative models.
\newblock In Corinna Cortes, Neil~D. Lawrence, Daniel~D. Lee, Masashi Sugiyama,
  and Roman Garnett (eds.), \emph{Advances in Neural Information Processing
  Systems 28: Annual Conference on Neural Information Processing Systems,
  Montreal, Quebec, Canada}, pp.\  3483--3491, 2015.

\bibitem[Song et~al.(2021{\natexlab{a}})Song, Durkan, Murray, and
  Ermon]{song2021maximum}
Yang Song, Conor Durkan, Iain Murray, and Stefano Ermon.
\newblock Maximum likelihood training of score-based diffusion models.
\newblock \emph{Advances in Neural Information Processing Systems},
  34:\penalty0 1415--1428, 2021{\natexlab{a}}.

\bibitem[Song et~al.(2021{\natexlab{b}})Song, Sohl-Dickstein, Kingma, Kumar,
  Ermon, and Poole]{song2021scorebased}
Yang Song, Jascha Sohl-Dickstein, Diederik~P Kingma, Abhishek Kumar, Stefano
  Ermon, and Ben Poole.
\newblock Score-based generative modeling through stochastic differential
  equations.
\newblock In \emph{International Conference on Learning Representations},
  2021{\natexlab{b}}.

\bibitem[Sprungk(2020)]{Sprungk_2020}
Björn Sprungk.
\newblock On the local {Lipschitz} stability of {Bayesian} inverse problems.
\newblock \emph{Inverse Problems}, 36\penalty0 (5), 2020.

\bibitem[Stuart(2010)]{stuart_2010}
A.~M. Stuart.
\newblock Inverse problems: A {Bayesian} perspective.
\newblock \emph{Acta Numerica}, 19:\penalty0 451–559, 2010.

\bibitem[Sullivan(2017)]{Sullivan_2017}
T.~J. Sullivan.
\newblock Well-posed {B}ayesian inverse problems and heavy-tailed stable
  quasi-banach space priors.
\newblock \emph{Inverse Problems in Imaging}, 11\penalty0 (5):\penalty0
  857--874, 2017.

\bibitem[Tallis(1963)]{Tallis1963}
G.~M. Tallis.
\newblock Elliptical and radial truncation in normal populations.
\newblock \emph{Annals of Mathematical Statistics}, 34:\penalty0 940--944,
  1963.

\bibitem[Tashiro et~al.(2021)Tashiro, Song, Song, and Ermon]{tashiro2021csdi}
Yusuke Tashiro, Jiaming Song, Yang Song, and Stefano Ermon.
\newblock {CSDI}: Conditional score-based diffusion models for probabilistic
  time series imputation.
\newblock In A.~Beygelzimer, Y.~Dauphin, P.~Liang, and J.~Wortman Vaughan
  (eds.), \emph{Advances in Neural Information Processing Systems}, 2021.

\bibitem[Teshima et~al.(2020)Teshima, Ishikawa, Tojo, Oono, Ikeda, and
  Sugiyama]{TITOIS2020}
Takeshi Teshima, Isao Ishikawa, Koichi Tojo, Kenta Oono, Masahiro Ikeda, and
  Masashi Sugiyama.
\newblock Coupling-based invertible neural networks are universal
  diffeomorphism approximators.
\newblock In \emph{Proceedings of the 34th International Conference on Neural
  Information Processing Systems}, Vancouver, BC, Canada, 2020.

\bibitem[Villani(2009)]{villani2009optimal}
C{\'e}dric Villani.
\newblock \emph{Optimal Transport: Old And New}, volume 338.
\newblock Springer, 2009.

\bibitem[Weed \& Bach(2019)Weed and Bach]{WB2019}
Jonathan Weed and Francis Bach.
\newblock {Sharp asymptotic and finite-sample rates of convergence of empirical
  measures in Wasserstein distance}.
\newblock \emph{Bernoulli}, 25\penalty0 (4A):\penalty0 2620 -- 2648, 2019.

\bibitem[Winkler et~al.(2019)Winkler, Worrall, Hoogeboom, and
  Welling]{winkler_cond_flows}
Christina Winkler, Daniel Worrall, Emiel Hoogeboom, and Max Welling.
\newblock Learning likelihoods with conditional normalizing flows.
\newblock \emph{arXiv preprint arXiv:1912.00042}, 2019.

\end{thebibliography}
\bibliographystyle{tmlr}

\appendix
\section{Example on the Robustness of the MAP and Posterior}\label{app:example}
We like to provide an example that illustrates the stability of the posterior distribution in contrast to the
MAP estimator and highlights the role of the MMSE estimator.

By the following lemma, see, e.g., \citep{GFO2017,HHS2023},
the posterior of a Gaussian mixture model given observations from a linear forward operator corrupted by white Gaussian noise
can be computed analytically.

\begin{lemma} \label{lem:mm}
Let $X \sim \sum_{k=1}^K w_k \mathcal N(m_k,\Sigma_k) \in \mathbb R^m$ be a Gaussian mixture random variable.
Suppose that 
$$
Y=A X+\Xi,
$$
where
$A: \R^m \rightarrow \R^n$ 
is a linear operator and $\Xi \sim N(0,\sigma^2 I_n)$. Then the posterior is also a Gaussian mixture
$$
P_{X|Y=y} \propto \sum_{k=1}^K \tilde w_k \mathcal N(\cdot|\tilde m_k,\tilde \Sigma_k)
$$
with
$$
\tilde \Sigma_k \coloneqq (\tfrac{1}{\sigma^2}A^\tT A+\Sigma_k^{-1})^{-1},
\qquad 
\tilde m_k \coloneqq \tilde\Sigma_k (\tfrac1{b^2}A^\tT y+\Sigma_k^{-1}\mu_k)
$$
and
$$
\tilde w_k \coloneqq \frac{w_k}{|\Sigma_k|^{1/2}} \exp\left(\frac12 (\tilde m_k \tilde \Sigma_k^{-1} \tilde m_k - m_k \Sigma_k^{-1} m_k)\right).
$$
\end{lemma}

Now, for some small $\varepsilon > 0$ we consider  
the random variable $X \in \mathbb R$ with  simple prior distribution 
$$P_X = \frac{1}{2} \mathcal{N}(-1,\varepsilon^2) + \frac{1}{2} \mathcal{N}(1,\varepsilon^2)$$
and observations from
$Y=X+\Xi$ 
with noise $\Xi \sim \mathcal{N}(0,\sigma^2)$. 
The  MAP estimator is given by
\begin{align}
x_{\mathrm{MAP}}(y) &\in \argmax_x p_{X|Y=y} (x) \\
&=
\argmin_x \frac{1}{2 \sigma^2} ( y - x )^2 - \log  \big( \frac{1}{2} (  e^{-\frac{1}{2 \varepsilon^2} ( x - 1 )^2} 
+  e^{-\frac{1}{2 \varepsilon^2} ( x + 1 )^2} )\big) \\
&= 
\argmin_x \frac{1}{2 \sigma^2} ( y - x )^2  + \frac{1}{2 \varepsilon^2} (x^2 + 1) - \log \left( \cosh \left(\frac{x}{\varepsilon^2} \right) \right).
\end{align}
The above minimization problem has a unique global minimizer for $y \not = 0$ which we computed numerically.
Figure \ref{fig:map_mmse_estimator} (top) shows the plot of the function $x_{\mathrm{MAP}}(y)$ for $\varepsilon^2 = 0.05^2$
and different values of $\sigma$. Clearly, small perturbations of $y$ near zero lead to qualitatively completely different $x$-values,
where a smaller noise level $\sigma$ lowers the distance between the values $x_{\mathrm{MAP}}(y)$ for $y>0$ and $y<0$.
In other words, the MAP estimator is not robust with respect to perturbations of the observations near zero.

\begin{figure*}[t]
\centering
\begin{subfigure}{.3\textwidth}
  \includegraphics[width=\linewidth]{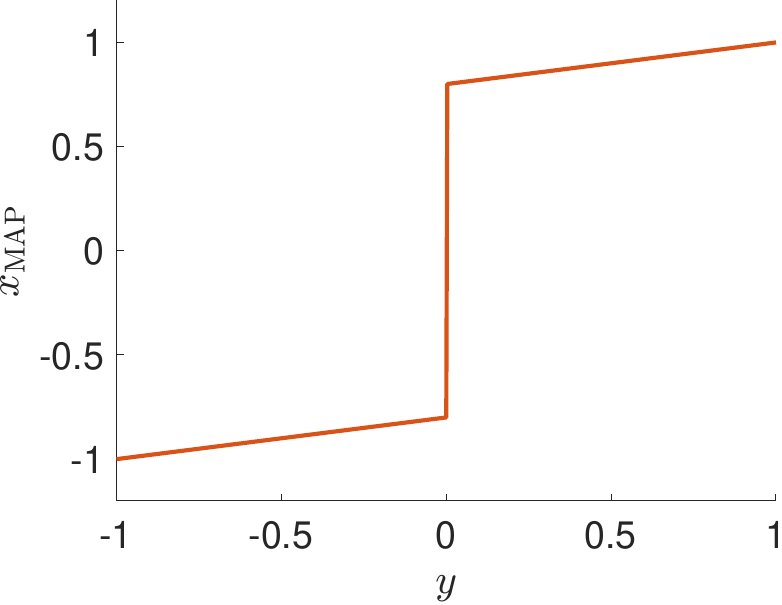}
\end{subfigure}%
\hspace{0.2cm}
\begin{subfigure}{.3\textwidth}
  \includegraphics[width=\linewidth]{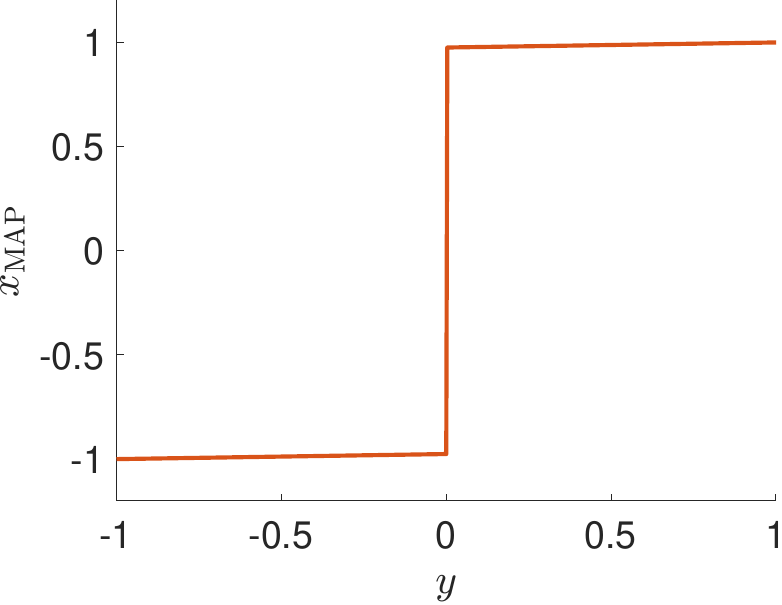}
\end{subfigure}%
\hspace{0.2cm}
\begin{subfigure}{.3\textwidth}
  \includegraphics[width=\linewidth]{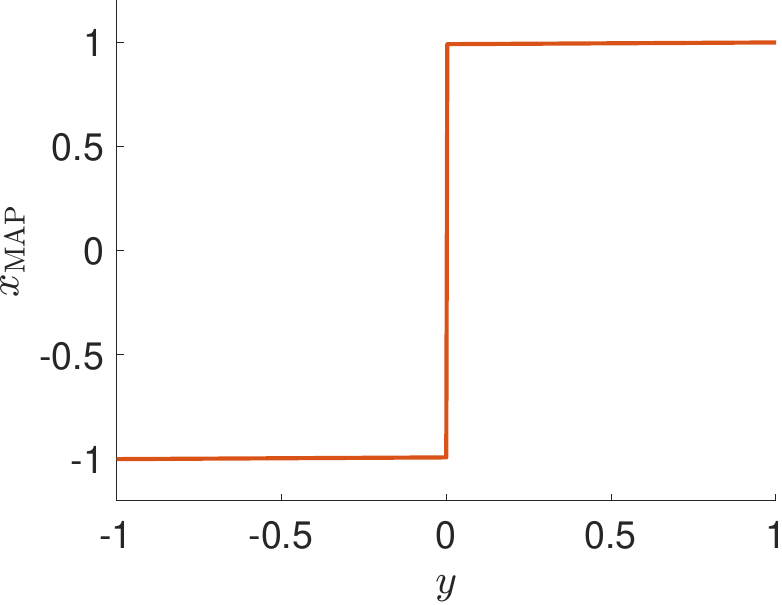}
\end{subfigure}%

\begin{subfigure}{.3\textwidth}
  \includegraphics[width=\linewidth]{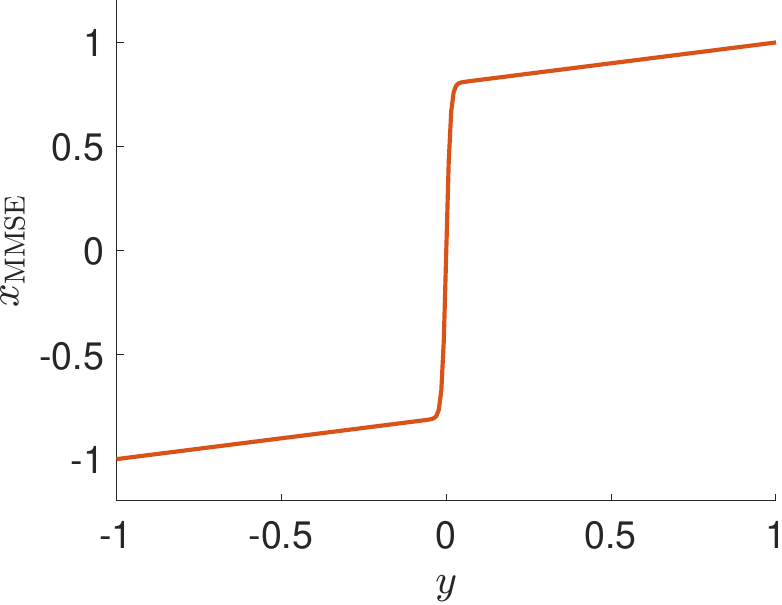}
  \caption*{$\sigma^2 = 0.01$}
\end{subfigure}%
\hspace{0.2cm}
\begin{subfigure}{.3\textwidth}
  \includegraphics[width=\linewidth]{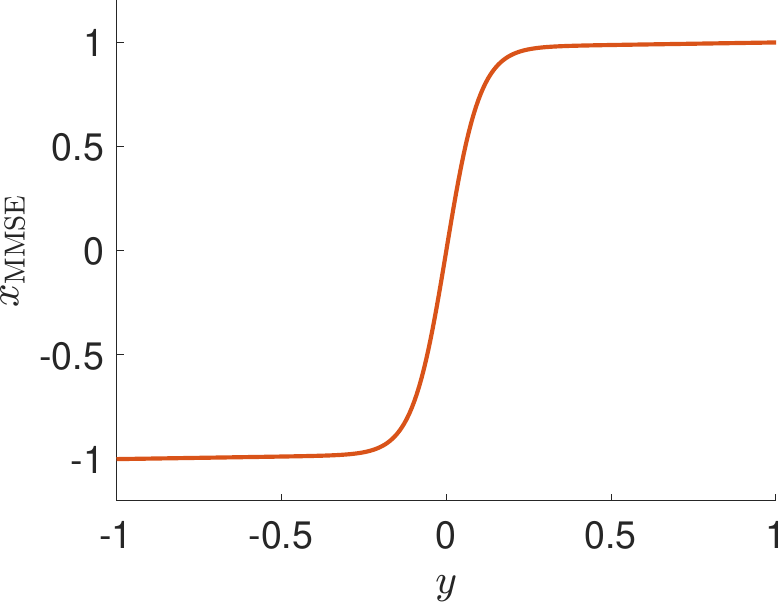}
  \caption*{$\sigma^2 = 0.1$}
\end{subfigure}%
\hspace{0.2cm}
\begin{subfigure}{.3\textwidth}
  \includegraphics[width=\linewidth]{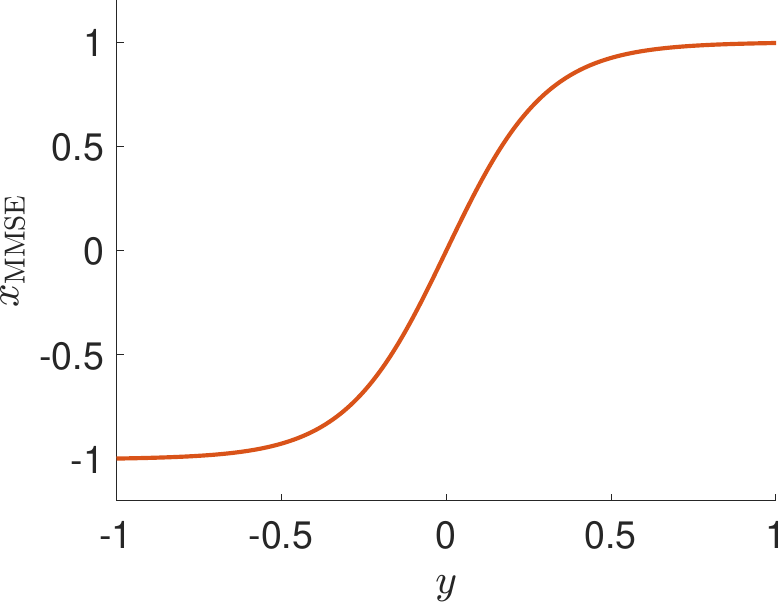}
  \caption*{$\sigma^2 = 0.3$}
\end{subfigure}%
\caption{
The MAP estimator (top) and the MMSE estimator (bottom) with respect to the observation $y$ for $\varepsilon^2 = 0.05^2$  and
different noise levels $\sigma^2$.
} \label{fig:map_mmse_estimator}
\end{figure*}

In contrast, using Lemma \ref{lem:mm}, we can compute the posterior 
\begin{align}
P_{X|Y=y} = \frac{1}{\tilde w_1 + \tilde w_2} (\tilde w_1 \mathcal{N}(\cdot | \tilde m_1 , \tilde \sigma^2) + \tilde w_2 \mathcal{N}(\cdot | \tilde m_2 \tilde \sigma^2) )
\end{align}
with 
\begin{align}
\tilde \sigma^2 &= \frac{\sigma^2 \varepsilon^2}{\sigma^2 + \varepsilon^2}, \quad 
\tilde m_1 = \frac{\varepsilon^2 y + \sigma^2}{\varepsilon^2 + \sigma^2}, \quad \tilde m_2 = \frac{\varepsilon^2 y - \sigma^2}{\varepsilon^2 + \sigma^2},
\\
\tilde w_1 &= \frac{1}{2 \varepsilon} \exp \Big(\frac{1}{2\varepsilon^2} \Big( \frac{(\varepsilon^2 y + \sigma^2)^2}{\sigma^2  (\varepsilon^2 + \sigma^2)} - 1 \Big) \Big), \quad 
\tilde w_2 = \frac{1}{2 \varepsilon} \exp \Big(\frac{1}{2\varepsilon^2} \Big( \frac{(\varepsilon^2 y - \sigma^2)^2}{\sigma^2  (\varepsilon^2 + \sigma^2)} - 1 \Big) \Big) .
\end{align}
Then the MMSE estimator is given by the expectation value of the posterior
\begin{align}
x_{\mathrm{MMSE}}(y) 
&= \argmin_{T} \mathbb E_{(x,y) \sim P_{(X,Y)}} \|x - T(y)\|^2  =  \mathbb E[X|Y=y] \label{loss_MMSE}\\
&= \int_{\mathbb R} x p_{X|Y = y} (x) \, \text{d} x\\
&= \frac{1}{\tilde w_1 + \tilde w_2} (\tilde w_1 \tilde m_1 + \tilde w_2 \tilde m_2)\\
&=
\frac{1}{\tilde w_1 + \tilde w_2}\frac{1}{\varepsilon ( \varepsilon^2 + \sigma^2)} e^{\frac{\varepsilon^2 y^2 - \sigma^2}{2 \sigma^2 (\varepsilon^2 + \sigma^2)} } \big(\varepsilon^2 y \cosh (\frac{y}{\varepsilon^2 + \sigma^2}) + \sigma^2 \sinh ( \frac{y}{\varepsilon^2 + \sigma^2}) \big).
\end{align}
In Figure \ref{fig:map_mmse_estimator} (bottom), we see that the MMSE estimator shows a smooth transition in particular for larger noise levels,
meaning that the estimator is robust against small perturbations of the observation near zero.
Note that in case of a Gaussian prior $X \sim \mathcal N (m, \Sigma)$ in $\mathbb R^m$ and white Gaussian noise, 
the MAP and MMSE estimators coincide.

\section{Local Lipschitz continuity of the generator for a latent space with infinite support} \label{app:infinite_support}

Here we show a weakened version of Lemma~\ref{gen_cont} leading to an arbitrary small additive constant. The main difference is the weaker assumption
 $\Vert \nabla_y G(y,z) \Vert \leq L_r$ for all $z \in \R^d$ with $\Vert z \Vert \le \tilde r$ and all $y \in \R^n$ with $\Vert y \Vert \le r$,
which is fulfilled for continuously differentiable generators. 
For this we use the so-called truncated normal distribution \citep{Horrace2005,Tallis1963}. Let 
$p_Z$ 
be the density of the standard normal distribution $P_Z = \mathcal{N}(0,I_n)$, then the density of the truncated normal distribution $P_Z^{\tilde r}$ is given by
\begin{align}
p_Z^{\tilde r} (z) = 
\begin{cases}
\frac{p_Z(z)}{\int_{B_{\tilde r}(0)} p_Z(z) \dx z}  = \frac{p_Z(z)}{C_{\tilde{r}}}, \quad &\text{if}~ \Vert z \Vert \le \tilde{r}, \\
0, &\text{else.}
\end{cases}
\end{align}

\begin{lemma}
Let $P_Z = \mathcal{N}(0,I_n)$ be the latent space. For any parameterized family of generative models $G$ with 
$\Vert \nabla_y G(y,z) \Vert \leq L_r$ for all $z \in \R^d$ with $\Vert z \Vert \le \tilde r$ and all $y \in \R^n$ with $\Vert y \Vert \le r$ for some $L_r > 0$ and some $r>0$, it holds 
$$W_1(G(y_1,\cdot)_{\#} P_Z, G(y_2,\cdot)_{\#} P_Z) \leq L_r \Vert y_1-y_2 \Vert + M_{\tilde r}$$ for all $y_1,y_2 \in \R^n$ with $\Vert y_1 \Vert, \Vert y_2 \Vert \le r$. The additive constant $M_{\tilde r}$ fulfills $M_{\tilde r} \to 0$ for $\tilde r \to \infty$.
\end{lemma}
\begin{proof}
Let $y_1,y_2 \in \R^n$ with $\Vert y_1 \Vert, \Vert y_2 \Vert \le r$, then it holds
\begin{align}
W_1(G(y_1,\cdot)_{\#} P_Z, G(y_2,\cdot)_{\#} P_Z) &\le W_1(G(y_1,\cdot)_{\#} P_Z, G(y_1,\cdot)_{\#} P_Z^{\tilde r}) 
+ W_1(G(y_1,\cdot)_{\#} P_Z^{\tilde r}, G(y_2,\cdot)_{\#} P_Z^{\tilde r}) \\
&\quad + W_1(G(y_2,\cdot)_{\#} P_Z^{\tilde r}, G(y_2,\cdot)_{\#} P_Z).
\end{align}
By the assumption on the generator $G$, Lemma~\ref{gen_cont} yields
\begin{align}
W_1(G(y_1,\cdot)_{\#} P_Z^{\tilde r}, G(y_2,\cdot)_{\#} P_Z^{\tilde r}) \le L_r \Vert y_1 - y_2 \Vert. 
\end{align}
Consequently, it suffices to show that for $y \in \R^n$ with $\Vert y \Vert \le r$ the term 
$W_1(G(y\,\cdot)_{\#} P_Z, G(y,\cdot)_{\#} P_Z^{\tilde r})$
vanishes for ${\tilde r} \to \infty$. By definition, it holds
\begin{align}
W_1(G(y,\cdot)_{\#} P_Z, G(y,\cdot)_{\#} P_Z^{\tilde r}) &= \max_{\mathrm{Lip}(\varphi)\le1} \int_{\R^d} \varphi(G(y,z)) \dx P_Z(z) - \int_{\R^d} \varphi(G(y,z)) \dx P_Z^{\tilde r}(z) \\
&= \max_{\mathrm{Lip}(\varphi)\le1} \int_{\R^d \setminus B_{\tilde r}(0)} \varphi(G(y,z)) \dx P_Z(z) + \int_{B_{\tilde r}(0)} \varphi(G(y,z)) \dx P_Z(z) \\
&\phantom{\max_{\mathrm{Lip}(\varphi)\le1}} \quad - \int_{B_{\tilde r}(0)} \varphi(G(y,z)) \dx P_Z^{\tilde r}(z) \\
&=\max_{\mathrm{Lip}(\varphi)\le1} \int_{\R^d \setminus B_{\tilde r}(0)} \varphi(G(y,z)) p_Z(z) \dx z \\
&\phantom{\max_{\mathrm{Lip}(\varphi)\le1}} \quad + \int_{B_{\tilde r}(0)} \varphi(G(y,z)) p_Z(z) (1 - \frac{1}{C_{\tilde r}})\dx z.
\end{align}
The first term vanishes exponentially in ${\tilde r}$ by the density $p_Z$, and for the second term note that $C_{\tilde r} \to 1$ for ${\tilde r} \to \infty$.\end{proof}
 
\section{Experimental demonstration of assumption \eqref{eq:expectation_W1} } \label{app:experiment}

Here we experimentally demonstrate that the expectation in \eqref{eq:expectation_W1} gets small when training a conditional generative network. For this, we  consider a conditional normalizing flow as an example,  consisting of three and ten Glow coupling blocks\footnote{available at \url{https://github.com/VLL-HD/FrEIA}} with fully connected subnets with one hidden layer and 64 and 512 nodes, respectively. As in Appendix~\ref{app:example}, we choose a two-dimensional Gaussian mixture model with six modes as prior distribution and additive Gaussian noise with standard deviation 0.5 for the noise model. The forward operator is chosen to be the identity. Then we train the conditional normalizing flow for 100000 optimizer steps using Adam \cite{KB2015} with a learning rate of $1e-4$ and a batch size of 1024. Note that we can analytically compute the posterior distribution by Lemma~\ref{lem:mm}.

\begin{figure}[t!]
\centering
\begin{subfigure}[t]{.45\textwidth}
\includegraphics[width=\linewidth]{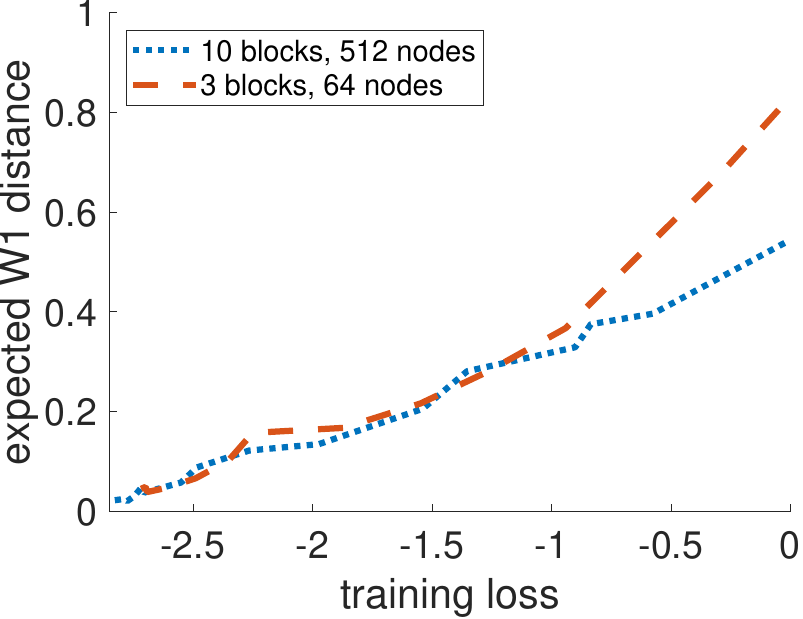}
\end{subfigure}
\caption{Expectation of the Wasserstein distance between posterior and pushforward of the generator with respect to the training loss of the conditional normalizing flow.}
\label{fig:W1vsKL}
\end{figure}

In Figure~\ref{fig:W1vsKL} we visualize the expectation in \eqref{eq:expectation_W1} with respect to the training loss of the conditional normalizing flow, which is, up to a constant, equal to the KL, see Section~\ref{sec:cond_flow}. Obviously, the expectation \eqref{eq:expectation_W1} gets small when minimizing the training loss of the conditional normalizing flow. We computed the Wasserstein distance using POT \footnote{available at \url{https://github.com/PythonOT/POT}} and draw 20000 samples from each distribution. Moreover, we discretized the expectation by drawing 30 observations.

\end{document}